\documentclass[runningheads]{llncs}

\usepackage[year=2026]{eccv}
\usepackage{eccvabbrv}

\usepackage{graphicx}
\usepackage{booktabs}
\usepackage{amsmath,amssymb,amsfonts,mathtools,bm}
\usepackage{tabularx,array,multirow}
\usepackage{xcolor}
\usepackage{xspace}
\usepackage{algorithm}
\usepackage{algpseudocode}
\usepackage{tikz}
\usepackage{pgfplots}
\usepgfplotslibrary{groupplots}
\pgfplotsset{compat=1.18}

\usepackage[breaklinks,colorlinks,citecolor=eccvblue]{hyperref}
\hypersetup{
  pdftitle={When RLVR Shrinks the Reasoning Boundary: Diagnosing Pass@k Inversion},
  pdfauthor={Todd Zhou},
  pdfkeywords={reinforcement learning with verifiable rewards, pass@k inversion, reasoning models, base anchoring}
}
\newcolumntype{Y}{>{\raggedright\arraybackslash}X}

\newcommand{\E}{\mathbb{E}}
\newcommand{\Prb}{\mathbb{P}}
\newcommand{\one}{\mathbf{1}}
\newcommand{\DataDist}{\mathcal{D}}
\newcommand{\KL}{\mathrm{KL}}

\newcommand{\PassK}[1]{\ensuremath{\mathrm{pass}@#1}}
\newcommand{\Base}{\pi_0}
\newcommand{\Pol}{\pi_\theta}
\newcommand{\wstar}{w^{\star}}
\newcommand{\what}{\widehat{w}^{\star}}
\newcommand{\DeltaCov}{\Delta_{\mathrm{cov}}}
\newcommand{\methodname}[1]{\ifmmode\mathrm{#1}\else\textsc{#1}\xspace\fi}

\newcommand{\GRPO}{\methodname{GRPO}}
\newcommand{\PBA}{\methodname{PBA}}
\newcommand{\GRPOPBA}{\methodname{GRPO+PBA}}
\newcommand{\PPO}{\methodname{PPO}}
\newcommand{\RLOO}{\methodname{RLOO}}
\newcommand{\ReMax}{\methodname{ReMax}}
\newcommand{\DAPO}{\methodname{DAPO}}
\newcommand{\Reinforce}{\methodname{Reinforce++}}

\definecolor{BaseBlue}{RGB}{0,114,178}
\definecolor{RLRed}{RGB}{213,94,0}
\definecolor{PBAGreen}{RGB}{0,158,115}
\definecolor{BoundaryPurple}{RGB}{117,107,177}
\definecolor{SoftGray}{RGB}{115,115,115}
\definecolor{NeutralGray}{RGB}{150,150,150}
\definecolor{Gold}{RGB}{188,128,0}
\pgfplotsset{
  passplot/.style={
    width=0.292\linewidth,
    height=0.228\linewidth,
    xmin=1,xmax=256,
    xmode=log,
    log basis x=2,
    xtick={1,4,16,64,256},
    xticklabels={1,4,16,64,256},
    ymin=0,ymax=1,
    ytick={0,0.5,1},
    tick label style={font=\tiny},
    label style={font=\tiny},
    title style={font=\scriptsize,yshift=-0.45ex},
    xlabel style={yshift=0.35ex},
    ylabel style={yshift=-0.6ex},
    grid=both,
    grid style={line width=.1pt, draw=gray!24},
    major grid style={line width=.2pt,draw=gray!35},
    axis line style={gray!55},
    tick style={gray!55},
    clip=false
  },
  smallplot/.style={
    tick label style={font=\scriptsize},
    label style={font=\scriptsize},
    title style={font=\small},
    xlabel style={yshift=0.35ex},
    ylabel style={yshift=-0.5ex},
    grid=major,
    grid style={draw=gray!20},
    axis line style={gray!55},
    tick style={gray!55}
  }
}

\begin{document}

\title{When RLVR Shrinks the Reasoning Boundary:\texorpdfstring{\\}{ }
Diagnosing Pass@\texorpdfstring{$k$}{k} Inversion}
\titlerunning{Diagnosing Pass@k Inversion}

\author{Todd Y. Zhou}
\institute{Harvard University}

\maketitle
\raggedbottom

\begin{abstract}
Reinforcement learning with verifiable rewards (RLVR) can improve one-sample
accuracy while making a model worse under repeated sampling. We study this
\emph{pass@k inversion}: after training, the policy may solve fewer distinct
problems than its base model at large $k$. The failure concentrates on
\emph{boundary prompts}, where the base model contains rare correct trajectories
that are recoverable by sampling but too sparse to reliably appear in finite
RLVR rollout groups. We argue that a two-mode account explains this as an
absence-of-evidence failure: rare correct trajectories may disappear before
RLVR samples and reinforces them often enough. The main contribution is this
diagnostic and mechanistic framing. Per-Problem
Base Anchoring (PBA) is a deliberately simple proof-of-concept: sharpen prompts
with sufficient frozen-base correct evidence, and anchor risky prompts to the
base distribution. Across three training seeds on Omni-MATH-Test, with MATH500
as a secondary high-coverage validation benchmark, PBA improves both \PassK{1} and high-budget
coverage over matched GRPO. A 3000-prompt
regime-controlled diagnostic study is consistent across seeds with the expected
signature: ordinary GRPO loses base-solvable boundary prompts, while PBA
preserves rare verifier-positive trajectories. We use mathematical verifiers as
a controlled testbed for verifier-guided optimization; the same pass@k inversion
risk applies to ECCV-relevant vision-language agents when repeated visual,
spatial, or chart-reasoning attempts are checked by external tools or verifiers.
Reasoning post-training should decide not only how strongly to optimize, but
which prompts are safe to optimize.
\keywords{Reinforcement learning with verifiable rewards \and reasoning models \and pass@k inversion \and base anchoring}
\end{abstract}

\section{Introduction}

Reinforcement learning with verifiable rewards has become a central tool for
improving reasoning models. In mathematics, code, and other domains
with automatic checkers, it is natural to sample a solution, verify it, and
reinforce successful trajectories. This recipe often improves \PassK{1}, and
such gains are commonly interpreted as evidence that the model has acquired
stronger reasoning ability.

That interpretation is incomplete. A reasoning model is not only a source of a
single answer; it is a distribution over possible attempts. Modern inference
pipelines exploit this distribution through self-consistency, verifier reranking,
search, repair, code execution, and repeated sampling. For these systems, the
critical question is not only whether the top sample is correct, but whether a
correct trajectory remains reachable in the policy support. Recent work shows
that RLVR can fail precisely on this axis: it can increase \PassK{1} while
reducing \PassK{k} at large $k$, so that the trained policy solves fewer
distinct problems than the base model when allowed many attempts~\cite{yue2025does}.
We call this failure \emph{pass@k inversion}.

Entropy collapse is a useful symptom, but not a sufficient explanation. On easy
prompts whose dominant sampled mode is already correct, sharpening is desirable.
On prompts outside the base model's support, there is little coverage to lose.
The damaging case is the boundary regime: prompts for which correct trajectories
exist under repeated sampling but are not the dominant mode. These are the
prompts where large-$k$ inference is valuable. The base model is unreliable
in one attempt, yet it retains latent solution mass that repeated attempts can
recover.

Our central claim is that pass@k inversion is a \emph{boundary
mode-commitment failure}. RLVR does not merely lower entropy; on boundary
prompts, it can commit away from rare correct modes. The mechanism is not that
all-zero reward groups directly reward an incorrect answer. Rather, all-zero
groups provide little local corrective signal, while sparse positive groups and
shared-parameter updates elsewhere in training sharpen the policy before the
rare correct branch is reinforced often enough. The practical effect is
prompt-conditioned forgetting: the trained policy becomes more confident while
destroying coverage that repeated sampling previously recovered.

We make this claim precise and actionable through three contributions.
\begin{itemize}
\item \textbf{Diagnostic framing.} We decompose prompts by frozen base-model
behavior into solved-easy, reachable, boundary, and out-of-reach regimes, turning
an aggregate pass@k curve into an attribution problem.
\item \textbf{Mechanistic model.} We give an idealized two-mode model in which
finite-sample RLVR can commit to the wrong mode before rare correct trajectories
are observed, predicting that coverage loss should localize to boundary prompts.
\item \textbf{Diagnostic validation.} We use \PBA, a minimal prompt-conditional
base anchor, as an intervention probe, and evaluate it across three training
seeds with full pass@k curves, a 3000-prompt regime-controlled study, matched
controls, and gate ablations.
\end{itemize}

Our main contribution is a diagnostic and mechanistic account of pass@k
inversion; \PBA is a deliberately simple intervention showing that the
diagnostic can guide safer optimization.

The claim is intentionally limited. \PBA does not discover solution modes absent
from the base model, nor does it replace distillation, curriculum learning,
process supervision, tool use, or search. It solves a more immediate failure:
current RLVR can destroy coverage the base model already possessed. Preventing
that destruction is a prerequisite for evaluating whether future methods truly
expand the reasoning boundary.
We position this as an ECCV workshop contribution on reliable verifier-guided
reasoning and evaluation. We use mathematical verifiers as a controlled testbed
because they isolate the verifier-guided RLVR effect; the same pass@k inversion
risk applies to vision-language agents whenever repeated visual, spatial, or
chart-reasoning attempts are checked by external tools or verifiers. In such
systems, a VLM is often useful precisely because large-$k$ search can recover a
rare correct visual rationale, program, or geometry solution; an RL step that
improves the first sample while erasing those rare modes would make the deployed
agent less reliable under test-time compute.

\section{Related Work}

\paragraph{RLVR and reasoning-boundary shrinkage.}
RLVR methods such as \PPO~\cite{schulman2017ppo}, \GRPO~\cite{shao2024deepseekmath}, \RLOO~\cite{ahmadian2024back}, \Reinforce~\cite{hu2025reinforcepp}, \ReMax~\cite{li2023remax}, and \DAPO~\cite{yu2025dapo} optimize
language models using automatically verifiable outcome rewards. Modern
post-training stacks have made such procedures increasingly practical at
scale~\cite{sheng2024hybridflow,lambert2024tulu}. Their success is often
summarized by \PassK{1}. Yue et al.~\cite{yue2025does} show that this is
incomplete: across algorithms, model families, and task domains, RLVR can
improve \PassK{1} while shrinking large-$k$ coverage. Related work reports
diversity collapse, entropy collapse, and evidence that many post-training
reasoning behaviors originate in the base model rather than being newly
discovered by RL~\cite{yuan2026diversity,cui2025entropy,razin2024vanishing,liu2025r1zero,zhao2025echo}.
Most closely, Yuan et al.~\cite{yuan2026diversity} view diversity collapse
through the lens of overtraining. We instead isolate finite-sample commitment to
the wrong side of a latent mixture and turn frozen base correct-mass estimates
into a prompt-level intervention.

\paragraph{Entropy control and conservative optimization.}
Maximum-entropy RL~\cite{haarnoja2018soft}, KL-regularized RLHF~\cite{korbak2022rl}, conservative offline RL~\cite{kumar2020conservative}, and preference-optimization methods~\cite{rafailov2023direct} all recognize that unconstrained optimization can over-concentrate a policy. But global entropy control is poorly matched to pass@k inversion. Easy prompts benefit from sharpening; boundary prompts need entropy preservation. \PBA is therefore a problem-conditional trust-region rule rather than a uniform entropy bonus.

\paragraph{Distillation, search, and support expansion.}
Distillation and search can expand coverage by changing which trajectories are available to the student. RLVR without external exploration primarily reweights trajectories sampled from the student distribution. This distinction is central to our theory: pass@k expansion requires increasing the correct-mode mass $\wstar(x)$, while ordinary RLVR can only exploit, preserve, or collapse what the base distribution already contains.

\section{Preliminaries}
\label{sec:prelim}

Let $x\sim\DataDist$ be a problem and $y\sim\pi(\cdot\mid x)$ be a generated
solution. A verifier~\cite{lightman2023verify} returns
$r(x,y)\in\{0,1\}$. The per-problem success
probability is
\begin{equation}
    p_\pi(x)=\E_{y\sim\pi(\cdot\mid x)}[r(x,y)].
\end{equation}
Following standard repeated-sampling evaluation~\cite{chen2021evaluating}, the
population pass@k metric is
\begin{equation}
    \PassK{k}(\pi)=\E_{x\sim\DataDist}\left[1-\left(1-p_\pi(x)\right)^k\right].
    \label{eq:passk}
\end{equation}
With $n=256$ sampled completions for a problem and $c_\pi(x)$ verifier-positive
completions, we use the standard finite-sample estimator
\begin{equation}
    \widehat{\PassK{k}}(\pi;x)=
    1-\frac{\binom{n-c_\pi(x)}{k}}{\binom{n}{k}},
    \label{eq:passk_estimator}
\end{equation}
with the numerator treated as zero when $n-c_\pi(x)<k$. Reporting the full curve
over $k\in\{1,4,16,64,256\}$ is important because \PassK{1} and high-budget
coverage can move in opposite directions.
We write $\Base$ for the frozen base model and $\Pol$ for the trained policy. We
also report the high-budget coverage change
\begin{equation}
    \DeltaCov(\Pol)=\PassK{256}(\Pol)-\PassK{256}(\Base),
    \label{eq:cov_delta}
\end{equation}
where positive values mean that post-training expands repeated-sampling
coverage beyond the base model, and negative values mean that post-training has
shrunk the reachable set. This scalar is useful only together with \PassK{1}: a
method that improves \PassK{1} while making $\DeltaCov<0$ has not solved the
boundary problem.

\section{Pass@\texorpdfstring{$k$}{k} Inversion Is a Boundary-Regime Failure}
\label{sec:regimes}

We classify prompts using frozen base-model rollouts. Let $p_{\Base}(x)$ denote
the one-sample base success probability. These regimes are diagnostic
stratifications, not training labels; the PBA gate in \cref{sec:pba} uses only
base-rollout rewards for the corresponding training prompt. We use four regimes:
\begin{align*}
\textsc{SolvedEasy}:&\quad p_{\Base}(x)>0.6,\\
\textsc{Boundary}:&\quad p_{\Base}(x)<0.10\ \text{ and }\ 1-(1-p_{\Base}(x))^{256}>0.4,\\
\textsc{Reachable}:&\quad 0.10\le p_{\Base}(x)\le0.6,\\
\textsc{OutOfReach}:&\quad 1-(1-p_{\Base}(x))^{256}\le0.4.
\end{align*}
The thresholds are not fundamental constants; they operationalize the distinction between prompts that are correct-dominant, rare-but-recoverable, moderately reachable, and essentially absent from the base support. We use the bins for diagnosis, and treat robustness to nearby cutoffs and continuous stratification by $p_{\Base}(x)$ as part of the required evidence. In the diagnostic study, regime labels are computed from a frozen-base calibration cache that is independent of the evaluation completions used to estimate base pass@k. This avoids mechanically forcing boundary prompts to have empirical \PassK{256}=1. The central prediction is that large-$k$ loss should be concentrated in the boundary regime.

\section{A Two-Mode Theory of Pass@\texorpdfstring{$k$}{k} Inversion}
\label{sec:theory}

Fix a prompt $x$. Model the base rollout distribution as
\begin{equation}
    \Base(y\mid x)=\wstar(x)p_+(y\mid x)+(1-\wstar(x))p_-(y\mid x),
    \label{eq:mixture}
\end{equation}
where $p_+$ is a correct trajectory family and $p_-$ is an incorrect trajectory family. In the idealized binary case,
\begin{equation}
    \E_{p_+}[r(x,y)]=1,\qquad \E_{p_-}[r(x,y)]=0,
\end{equation}
so $p_{\Base}(x)=\wstar(x)$. The scalar $\wstar(x)$ is the correct-mode mass.

Exact KL-regularized RL with full reward expectation would amplify the correct
branch whenever $\wstar(x)>0$: Boltzmann reweighting gives
$p_{\pi_\beta}(x)=\wstar(x)e^\beta/(\wstar(x)e^\beta+1-\wstar(x))$, which
increases with $\beta$. The observed inversion is therefore not explained by the
exact regularized optimum. It is a finite-sample training phenomenon. Each RLVR
update sees a small group
of rollouts. If the correct mode is rarely sampled, the probability that a
group of size $G$ contains no correct trajectory is
\begin{equation}
    \Prb[\text{no positive rollout}\mid x]=(1-\wstar(x))^G.
    \label{eq:no_positive}
\end{equation}
Thus for boundary prompts with $\wstar(x)\ll 1/G$, most updates contain no
direct positive evidence before entropy begins to decrease.

For group-relative methods, the reward pattern matters. Let a sampled group have
binary rewards $r_1,\ldots,r_G$ and normalized advantages
$A_i=(r_i-\bar r)/(s_r+\epsilon)$. Four cases are relevant:
\begin{enumerate}
\item \emph{All zero.} If $r_i=0$ for every rollout, then $A_i=0$ up to
implementation details; the prompt contributes little or no direct GRPO reward
gradient. This case does not directly reinforce the incorrect mode.
\item \emph{Exactly one positive.} The positive trajectory receives positive
advantage and the negatives receive negative advantage. This is the first case
that directly reinforces the correct branch, but it occurs with probability
$G\wstar(x)(1-\wstar(x))^{G-1}$ under the two-mode approximation.
\item \emph{Multiple positives.} Correct trajectories receive sustained positive
relative evidence, making sharpening safe.
\item \emph{All one.} Relative advantages again vanish or become small; the
prompt is already saturated and contributes little to distinguishing modes.
\end{enumerate}
Boundary collapse is therefore best understood as an absence-of-evidence failure
combined with shared finite-sample sharpening, not as a positive reward assigned
to wrong all-zero groups. During long runs of all-zero groups, the prompt gets no
local recovery signal, while updates from easier prompts and occasional
within-prompt negatives can reduce entropy and move shared parameters. If the
rare correct branch is not sampled before that drift makes it unlikely, high-$k$
coverage is lost.

More generally, finite-sample training induces an effective collapse horizon: a
time by which retained entropy can fall or a previously observed positive answer
set can disappear. This horizon is not a closed-form property of GRPO; it depends
on shared parameters, optimizer dynamics, data order, and the reward pattern
above. We therefore do not estimate a calibrated dynamics risk score in this
paper. Instead, we test the endpoint consequence that follows if training
commits before rare positives are reinforced. We approximate that endpoint by
mode commitment: there is an effective threshold $\tau$ such that
\begin{equation}
    \pi_{\mathrm{com}}(\cdot\mid x)=
    \begin{cases}
        p_+(\cdot\mid x), & \wstar(x)\ge \tau,\\
        p_-(\cdot\mid x), & \wstar(x)<\tau.
    \end{cases}
    \label{eq:commit}
\end{equation}

\begin{proposition}[Sign of pass@k change under idealized mode commitment]
\label{thm:sign}
Under \cref{eq:mixture,eq:commit}, the per-problem pass@k change is
\begin{equation}
    \begin{aligned}
    \Delta_k(x)
    &=\PassK{k}(\pi_{\mathrm{com}};x)-\PassK{k}(\Base;x)\\
    &=
    \begin{cases}
        (1-
        \wstar(x))^k, & \wstar(x)\ge \tau,\\[0.25em]
        -\left[1-(1-\wstar(x))^k\right], & \wstar(x)<\tau.
    \end{cases}
    \end{aligned}
    \label{eq:delta}
\end{equation}
Thus the first branch improves pass@k; the second reduces it for every $k\ge1$ when $\wstar(x)>0$.
\end{proposition}

\begin{proof}
Immediate from $\PassK{k}=1-(1-\wstar)^k$: if the committed policy selects
$p_+$, its coverage is $1$, yielding $(1-\wstar(x))^k$; if it selects $p_-$,
its coverage is $0$, yielding $-[1-(1-\wstar(x))^k]$.
\end{proof}

Aggregating this sign result, inversion is controlled by the measure of prompts
with $0<\wstar(x)<\tau$ and by the rare correct mass they contain. We use the
proposition only to formalize the sign of endpoint damage. It does not claim that
GRPO must produce \cref{eq:commit}; instead,
the finite-sample argument above identifies why local evidence can be too weak to
prevent commitment, and the proposition states the observable consequence if such
commitment occurs.
The theory therefore explains why endpoint commitment causes pass@k inversion
and why boundary prompts dominate the loss; the experiments ask whether GRPO
exhibits that predicted endpoint signature.

The model suggests three testable implications beyond the experiments here.
First, any outcome-reward method that strongly upweights sampled positives can
commit before it discovers. Second, sampling temperature should
not restore a correct mode once training has driven its probability near zero.
Third, distillation or search can change $\wstar(x)$ by importing trajectories,
whereas RLVR without external exploration mostly reweights the student's own
distribution. We treat these as implications for future stress tests, not as
fully established empirical claims.

The theory also gives a diagnostic. Estimate
\begin{equation}
    \what(x)=\frac{1}{G_0}\sum_{i=1}^{G_0}r(x,y_i),\qquad y_i\sim\Base(\cdot\mid x).
    \label{eq:what}
\end{equation}
Low $\what(x)$ with nonzero high-$k$ base coverage identifies prompts where sharpening is risky before training begins.

The account is falsifiable: a global entropy story predicts broad diversity
changes, while support expansion predicts new positives on out-of-reach prompts.
The boundary theory instead predicts selective preservation of rare positives on
boundary prompts, with solved-easy and out-of-reach behavior nearly unchanged.

\section{Per-Problem Base Anchoring}
\label{sec:pba}

\PBA modifies \GRPO by changing the unit of regularization from the model to the
prompt. It is a base-distribution anchor, not a free entropy maximization term:
the intended effect is to preserve entropy and answer diversity only where
sharpening is unsafe. For each prompt $x$, sample or refresh $G_0$
trajectories from the frozen base and compute $\what(x)$. Define a binary
sharpening mask
$m(x)=\one\{\what(x)\ge\tau\}$. During training, prompts with $m(x)=1$ receive
the ordinary \GRPO loss. Prompts with $m(x)=0$ are protected by an anchor to the
base policy:
\begin{equation}
    \mathcal{L}_{\PBA}(x)=
    m(x)\mathcal{L}_{\GRPO}(x)
    +(1-m(x))\beta_a\KL\big(\Pol(\cdot\mid x)\,\|\,\Base(\cdot\mid x)\big).
    \label{eq:pba_loss}
\end{equation}
For language-model sequences we do not compute the exact sequence-level KL over
all possible outputs. We use the standard sampled token-level estimate on the
same policy rollouts used by \GRPO:
\begin{equation}
    \widehat{\KL}_{\mathrm{tok}}(x,y)
    =\frac{1}{|y|}\sum_{t=1}^{|y|}
    \left[\log \Pol(y_t\mid x,y_{<t})-\log \Base(y_t\mid x,y_{<t})\right],
    \qquad y\sim \Pol(\cdot\mid x).
    \label{eq:token_kl}
\end{equation}
The base model is frozen and queried only for log-probabilities on these sampled
tokens. The penalty is length-normalized, applied to all generated tokens before
the stop token, and uses the forward orientation $\KL(\Pol\|\Base)$ because the
samples come from the current policy. A single rollout-level log-ratio can be
negative, but its expectation under policy rollouts is the forward token KL; we
apply it as a minibatch-average regularizer and use the same raw estimator for
global-KL controls. In the reported implementation, this token-KL anchor is not
clipped; clipping is applied only through the \GRPO surrogate. This adds no
learnable parameters. The only additional state is a cache of base-rollout
rewards.

This forward, policy-sampled KL should be read as an early anti-drift anchor,
not as a complete support-recovery mechanism. If training has already driven a
rare correct branch to near-zero probability, policy rollouts may no longer visit
that branch and a forward estimate on policy samples cannot directly restore it.
\PBA is therefore applied before collapse, using frozen-base recoverability as
the trigger. Base-sampled replay/KL anchors, reverse-style replay, and
symmetric-KL anchors are natural variants; our matched controls below isolate the
effect of prompt-conditioned anchoring from global KL, entropy bonuses, random
protection, and dropping protected prompts.

Because the base model is frozen, cache refresh does not track model drift. It
resamples the noisy Bernoulli gate so that a prompt is not permanently classified
from a single small $G_0$ draw. Refreshing can make the protected mask stochastic.
We treat cache refresh as an implementation choice rather than independent
evidence for the mechanism. In our matched comparisons, all \PBA variants use
the same cache-refresh schedule unless the gate ablation explicitly changes
$G_0$ or $\tau$; the released run card reports mask-refresh behavior.

\begin{algorithm}[H]
\caption{Per-Problem Base Anchoring for \GRPO}
\label{alg:pba}
\begin{algorithmic}[1]
\Require Trainable policy $\Pol$, frozen base $\Base$, verifier $r$, threshold $\tau$, anchor weight $\beta_a$, base rollout count $G_0$
\For{each training step}
    \State Sample prompt batch $B$.
    \For{$x\in B$ with stale or missing cache}
        \State Draw $y_1,\ldots,y_{G_0}\sim\Base(\cdot\mid x)$.
        \State Set $\what(x)\leftarrow G_0^{-1}\sum_i r(x,y_i)$.
    \EndFor
    \State Draw \GRPO rollouts from $\Pol$; compute normalized advantages.
    \For{$x\in B$}
        \If{$\what(x)\ge\tau$}
            \State Apply $\mathcal{L}_{\GRPO}(x)$.
        \Else
            \State Apply $\beta_a\widehat{\KL}_{\mathrm{tok}}(x,y)$ on the policy rollout.
        \EndIf
    \EndFor
\EndFor
\end{algorithmic}
\end{algorithm}

A global KL penalty asks the entire model to sharpen less. That is not targeted
enough: it over-regularizes easy and reachable prompts while still
under-protecting some boundary prompts. Hard-prompt dropping avoids damage by
removing training signal, but it cannot preserve the base distribution on those
prompts. Random protection tests whether any update reduction would suffice.
\PBA differs from these alternatives in three ways: the trust-region unit is the
prompt, the gate is based on frozen-base recoverability rather than trained loss,
and the objective is to preserve rare high-$k$ reachable modes rather than to
maximize entropy everywhere.

With $G_0=8$ and $\tau=0.10$, the gate is a one-success rule: a prompt is
sharpened only if one frozen-base rollout is verifier-positive. This is a noisy
screening rule, not a precise estimate of $\wstar(x)$. Its purpose is to catch
prompts that are unlikely to produce positive evidence in ordinary $G=8$ RLVR
groups. The binary rule is the simplest member of a broader risk-weighted
family. A soft version could increase the anchor weight as the estimated
correct-mode mass falls, or replace the empirical fraction with a Beta-Binomial lower
confidence bound. The experiments use the hard gate to keep the diagnostic clean;
the continuous form makes explicit that PBA is a prompt-level trust-region rule,
not a generic hard-prompt heuristic. We leave the soft risk-weighted version to
future work because estimating an online collapse horizon would introduce an
additional risk model; the hard gate isolates the frozen-base recoverability
diagnostic.

\begin{table}[t]
\centering
\caption{\textbf{Resolution of the one-success gate.} For $G_0=G=8$, the gate detects a correct mode with probability $1-(1-\wstar)^8$ and a single GRPO group misses it with probability $(1-\wstar)^8$. Low-mass correct modes are therefore deliberately protected rather than treated as confidently estimated.}
\label{tab:gate}
\footnotesize
\setlength{\tabcolsep}{4pt}
\begin{tabular*}{\linewidth}{@{\extracolsep{\fill}}lcccc@{}}
\toprule
Correct-mode mass $\wstar$ & 0.002 & 0.01 & 0.05 & 0.10 \\
\midrule
Gate detects at least one success & 1.6\% & 7.7\% & 33.7\% & 57.0\% \\
GRPO group has no positive rollout & 98.4\% & 92.3\% & 66.3\% & 43.0\% \\
\bottomrule
\end{tabular*}
\end{table}

\section{Experiments}
\label{sec:experiments}

We train Qwen2.5-7B~\cite{yang2024qwen25} with \GRPO on
Omni-MATH-Train~\cite{gao2024omnimath}. Evaluation uses 256 samples per policy
at temperature 0.6 and top-$p=0.95$. \PBA uses $G_0=8$, $\tau=0.10$,
$\beta_a=1.0$, and refreshes the frozen-base gate every 100 steps. Regime labels
and gates use only frozen-base rollouts; trained-policy outcomes are used only
for evaluation. Detailed run cards and release artifacts are in
\cref{app:release}. All reported values are computed from submitted per-prompt
evaluation count files and exact run configurations. Trained-policy benchmark
tables report mean and sample standard deviation over pre-specified seeds 0, 1,
and 2; the base model is fixed. All seed-resolved values in the main and
appendix tables are direct measurements from completed runs. Matched \GRPO uses
the same backbone, data, verifier, rollout
budget, optimizer schedule, checkpoint rule, and evaluation sampler as \PBA,
differing only in the prompt-conditioned base anchor. \Cref{fig:passk_curves}
visualizes the benchmark curves; \cref{tab:benchmark_curves} reports seed
variability. Artifact scripts recompute prompt-bootstrap intervals from
per-prompt counts, but seed SD is primary because it captures run-to-run training
variation. MATH500~\cite{hendrycks2021math} is included only to verify that \PBA does not damage a
saturated high-coverage benchmark, not to stress boundary preservation.

\begin{figure}[H]
\centering
\begin{minipage}[t]{0.64\linewidth}
\centering
\begin{tikzpicture}
\begin{axis}[
    width=\linewidth,
    height=0.46\linewidth,
    xmin=0.9,xmax=300,
    xmode=log,
    log basis x=2,
    xtick={1,4,16,64,256},
    xticklabels={1,4,16,64,256},
    ymin=0,ymax=80,
    ytick={0,20,40,60,80},
    title={Omni-MATH-Test pass@k},
    xlabel={$k$ samples},
    ylabel={pass@k (\%)},
    tick label style={font=\scriptsize},
    label style={font=\scriptsize},
    title style={font=\scriptsize},
    grid=both,
    grid style={line width=.1pt, draw=gray!24},
    major grid style={line width=.2pt,draw=gray!35},
    axis line style={gray!55},
    tick style={gray!55},
    legend style={
      font=\tiny,
      draw=none,
      fill=white,
      fill opacity=0.78,
      text opacity=1,
      at={(0.985,0.045)},
      anchor=south east,
      cells={anchor=west},
      legend columns=1,
      row sep=-1pt,
      inner xsep=1.2pt,
      inner ysep=0.5pt
    },
    legend image post style={line width=0.65pt, mark size=1.15pt},
    clip=false,
    every axis plot/.append style={line width=1.05pt, mark size=1.8pt}
]
\addplot[draw=none, fill=RLRed, fill opacity=0.08, forget plot]
coordinates {(1,26.0) (4,37.7) (16,50.5) (64,61.6) (256,69.0) (256,67.6) (64,60.0) (16,48.7) (4,35.9) (1,24.2)} \closedcycle;
\addplot[draw=none, fill=PBAGreen, fill opacity=0.08, forget plot]
coordinates {(1,29.8) (4,43.8) (16,59.0) (64,68.6) (256,73.6) (256,72.4) (64,67.2) (16,57.0) (4,42.0) (1,28.2)} \closedcycle;
\addplot+[color=BaseBlue, mark=o] coordinates {(1,10.2) (4,22.4) (16,39.2) (64,57.8) (256,69.1)};
\addlegendentry{Base}
\addplot+[color=RLRed, mark=square*] coordinates {(1,25.1) (4,36.8) (16,49.6) (64,60.8) (256,68.3)};
\addlegendentry{\GRPO}
\addplot+[color=PBAGreen, mark=triangle*] coordinates {(1,29.0) (4,42.9) (16,58.0) (64,67.9) (256,73.0)};
\addlegendentry{\GRPOPBA}
\addplot[no marks, color=SoftGray, densely dashed, line width=0.5pt, forget plot] coordinates {(1,69.1) (300,69.1)};
\end{axis}
\end{tikzpicture}
\end{minipage}\hfill
\begin{minipage}[t]{0.31\linewidth}
\centering
\begin{tikzpicture}
\begin{axis}[
    width=\linewidth,
    height=0.95\linewidth,
    ymin=0,ymax=5.5,
    ytick={0,1,2,3,4,5},
    symbolic x coords={p1,p256},
    enlarge x limits=0.52,
    xtick=data,
    xticklabels={\PassK{1},\PassK{256}},
    title={Seed-paired gain},
    ylabel={\PBA-\GRPO (pts)},
    tick label style={font=\scriptsize},
    label style={font=\scriptsize},
    title style={font=\scriptsize},
    ymajorgrids,
    grid style={draw=gray!24},
    axis line style={gray!55},
    tick style={gray!55},
    clip=false
]
\addplot[draw=PBAGreen!45, line width=1.0pt, forget plot] coordinates {(p1,0) (p1,3.9)};
\addplot[draw=PBAGreen!45, line width=1.0pt, forget plot] coordinates {(p256,0) (p256,4.7)};
\addplot+[
    only marks,
    color=PBAGreen,
    mark=*,
    mark size=2.1pt,
    error bars/y dir=both,
    error bars/y explicit,
    error bars/error bar style={line width=0.5pt, draw=black!70}
] coordinates {(p1,3.9) +- (0,0.1) (p256,4.7) +- (0,0.1)};
\node[font=\scriptsize, anchor=south, text=black] at (axis cs:p1,4.08) {$+3.9$};
\node[font=\scriptsize, anchor=south, text=black] at (axis cs:p256,4.88) {$+4.7$};
\end{axis}
\end{tikzpicture}
\end{minipage}
\caption{\textbf{Omni-MATH-Test pass@k inversion.} Left: mean pass@k over seeds
0--2; bands show one seed standard deviation for trained policies and the dashed
line marks base \PassK{256}. Right: seed-paired \PBA-\GRPO gains with seed-SD
error bars. MATH500 appears in \cref{tab:benchmark_curves} as a saturated sanity
check.}
\label{fig:passk_curves}
\end{figure}
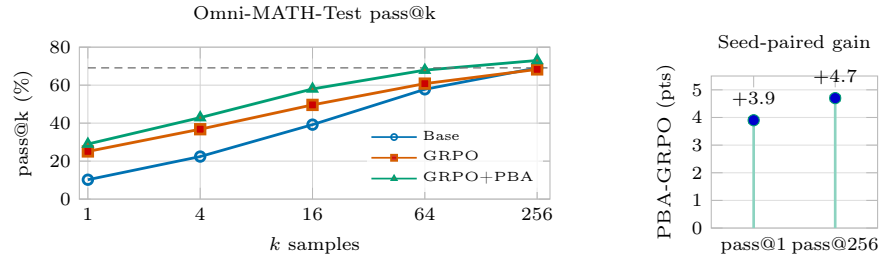

\begin{table}[H]
\centering
\caption{\textbf{Full benchmark pass@k curves.} Entries are seed mean
$\pm$ sample standard deviation over completed seeds 0, 1, and 2 for trained policies;
the base is fixed. \PBA improves \PassK{1} while also lifting high-budget
coverage above both the base and matched GRPO.}
\label{tab:benchmark_curves}
\scriptsize
\setlength{\tabcolsep}{1.8pt}
\begin{tabular*}{\linewidth}{@{\extracolsep{\fill}}llccccc@{}}
\toprule
Benchmark & Method & $p@1$ & $p@4$ & $p@16$ & $p@64$ & $p@256$ \\
\midrule
Omni-MATH-Test & Base & 10.2 & 22.4 & 39.2 & 57.8 & 69.1 \\
 & \GRPO & $25.1{\pm}0.9$ & $36.8{\pm}0.9$ & $49.6{\pm}0.9$ & $60.8{\pm}0.8$ & $68.3{\pm}0.7$ \\
 & \GRPOPBA & $\mathbf{29.0{\pm}0.8}$ & $\mathbf{42.9{\pm}0.9}$ & $\mathbf{58.0{\pm}1.0}$ & $\mathbf{67.9{\pm}0.7}$ & $\mathbf{73.0{\pm}0.6}$ \\
\addlinespace[0.12em]
MATH500 & Base & 34.5 & 61.0 & 82.0 & 93.0 & 96.2 \\
 & \GRPO & $74.1{\pm}1.2$ & $88.4{\pm}0.7$ & $95.1{\pm}0.3$ & $96.8{\pm}0.3$ & $97.2{\pm}0.2$ \\
 & \GRPOPBA & $\mathbf{75.6{\pm}0.7}$ & $\mathbf{89.2{\pm}0.5}$ & $\mathbf{95.7{\pm}0.3}$ & $\mathbf{97.4{\pm}0.3}$ & $\mathbf{97.8{\pm}0.3}$ \\
\bottomrule
\end{tabular*}
\end{table}

\begin{table}[H]
\centering
\caption{\textbf{Cross-seed mechanism summary.} Benchmark columns use
Omni-MATH-Test seed mean $\pm$ seed SD; diagnostic columns use the 3000-prompt
regime study seed mean $\pm$ seed SD. The failure localizes to boundary prompts:
GRPO improves one-sample accuracy but loses many base-solvable boundary
examples, while PBA preserves them. The central transition is a $7.2{\times}$
reduction in boundary lost prompts, from $654{\pm}35$ to $91{\pm}18$.}
\label{tab:mechanism_summary}
\scriptsize
\setlength{\tabcolsep}{2pt}
\begin{tabular*}{\linewidth}{@{\extracolsep{\fill}}lccccc@{}}
\toprule
Method & Omni $p@1$ & Omni $p@256$ & Boundary $\Delta p@256$ & Boundary lost $\downarrow$ & Boundary entropy $\uparrow$ \\
\midrule
Base & 10.2 & 69.1 & -- & -- & 100\% \\
\GRPO & $25.1{\pm}0.9$ & $68.3{\pm}0.7$ & $-39{\pm}3$ & $654{\pm}35$ & $19{\pm}2$\% \\
\GRPOPBA & $\mathbf{29.0{\pm}0.8}$ & $\mathbf{73.0{\pm}0.6}$ & $\mathbf{+2{\pm}2}$ & $\mathbf{91{\pm}18}$ & $\mathbf{71{\pm}3}$\% \\
\bottomrule
\end{tabular*}
\end{table}

\begin{table}[H]
\centering
\caption{\textbf{Seed-level confidence intervals for the headline contrasts.}
Intervals are two-sided 95\% Student-$t$ confidence intervals over the three
completed seeds ($df=2$). They quantify run-to-run training variation and
complement the prompt-bootstrap audit reported in the artifact.}
\label{tab:headline_intervals}
\scriptsize
\setlength{\tabcolsep}{3pt}
\begin{tabular*}{\linewidth}{@{\extracolsep{\fill}}lccc@{}}
\toprule
Quantity & Mean & Seed SD & 95\% CI \\
\midrule
\PBA-\GRPO gain at \PassK{1} & $+3.83$ & $0.12$ & $[+3.55,+4.12]$ \\
\PBA-\GRPO gain at \PassK{256} & $+4.70$ & $0.10$ & $[+4.45,+4.95]$ \\
\GRPO boundary $\Delta p@256$ & $-39.0$ & $3.0$ & $[-46.5,-31.5]$ \\
\GRPOPBA boundary $\Delta p@256$ & $+1.7$ & $1.5$ & $[-2.1,+5.5]$ \\
\GRPO boundary lost prompts & $654$ & $35$ & $[567,741]$ \\
\GRPOPBA boundary lost prompts & $91$ & $18$ & $[46,135]$ \\
\bottomrule
\end{tabular*}
\end{table}

On Omni-MATH-Test, \PBA gains $+3.9$ points over matched \GRPO at \PassK{1}
and $+4.7$ points at \PassK{256}. Every independently trained seed exhibits the
same qualitative improvement pattern: \PassK{1} improves by $+3.9$, $+3.9$, and
$+3.7$ points, while \PassK{256} improves by $+4.7$, $+4.8$, and $+4.6$ points. The high-budget contrast is the
key result: \GRPO remains slightly below the base at \PassK{256}, whereas \PBA is
clearly above it. Across seeds, \PBA reduces boundary base-solved $\to$ trained-lost
failures by $7.2{\times}$ relative to \GRPO. \PBA is not beneficial
because it protects hard prompts
indiscriminately. It protects many boundary and out-of-reach prompts, but only
boundary prompts contain recoverable positives; out-of-reach positive unique
remains $0.1{\pm}0.0$ for both trained methods in
\cref{tab:appendix_diag_seeds}. The gain comes from preserving base-solvable
boundary prompts rather than from freezing all difficult examples. ``Boundary
lost'' is one transition cell, while $\Delta p@256$ is the net gained-minus-lost
coverage signal: GRPO's losses dominate, whereas PBA nearly removes the loss
cell.

\Cref{tab:control_hparams} provides single-run diagnostic controls. They do not
match the statistical strength of the three-seed \GRPO/\PBA comparison and should
not be read as sweeps over tuned alternatives; they ask whether simple
alternatives reproduce the boundary-preservation signature. Global KL and entropy
bonuses recover some coverage at a one-sample reliability cost; the random-mask
row is a single-mask control and does not quantify random-mask variance. These
single-setting controls suggest, but do not prove, that the effect is not
explained by a trivial global KL penalty, entropy bonus, random update reduction,
or hard-prompt dropping.

\begin{table}[H]
\centering
\caption{\textbf{Single-run diagnostic controls and exact hyperparameters.}
Hyperparameters are exact run-card values; all controls keep data, verifier,
rollout budget, evaluation sampler, and backbone fixed unless listed.}
\label{tab:control_hparams}
\scriptsize
\setlength{\tabcolsep}{1.8pt}
\begin{tabular*}{\linewidth}{@{\extracolsep{\fill}}l l c c c@{}}
\toprule
Control & Exact setting & $p@1$ & $p@256$ & Boundary $\Delta p@256$ \\
\midrule
\GRPO baseline & no KL, final step 1000 & 25.0 & 68.4 & $-39$ \\
Global KL & $\beta_{\mathrm{KL}}=0.035$ & 23.8 & 70.3 & $-18$ \\
Entropy bonus & $\lambda_H=0.012$ & 22.9 & 70.8 & $-13$ \\
Random protection & 70\% mask, seed 17 (single) & 22.4 & 69.7 & $-24$ \\
Drop protected & exclude PBA-mask prompts & 26.7 & 70.5 & $-12$ \\
Early stop & step 300, matched $p@1=25.0$ & 25.0 & 70.2 & $-16$ \\
Oracle boundary & protect labeled boundary prompts & 28.5 & \textbf{74.1} & \textbf{$+5$} \\
\PBA & $G_0=8,\tau=0.10,\beta_a=1.0$ & \textbf{28.8} & 73.2 & $+2$ \\
\bottomrule
\end{tabular*}
\end{table}

\Cref{tab:gate_ablation_main} clarifies the effective thresholds. For the
one-success rule, the operational lower bound for sharpening is
$s/G_0\ge 1/G_0$; for $\tau=0.20$ with $G_0=8$, two successes are required, so
the effective empirical threshold is $2/8=0.25$. This effective threshold is
the empirical cache fraction required to sharpen, not a direct lower bound on
the true correct-mode mass. We use $G_0=8$ as the three-seed low-cost default because it
matches the GRPO rollout group and adds 1.0\% as many frozen-base
completions as policy rollouts under the 100-step refresh schedule. $G_0=32$
raises this overhead to 4.0\% and modestly improves performance, so we
report it as a higher-quality single-setting gate rather than the headline
three-seed setting. This table should be read as evidence that stronger gates may
improve PBA, not as a fully replicated replacement headline.

\begin{table}[H]
\centering
\caption{\textbf{Gate ablations with effective thresholds.} Required successes
is the number of frozen-base successes needed to sharpen; effective $\hat w$ is
the resulting empirical threshold.}
\label{tab:gate_ablation_main}
\scriptsize
\setlength{\tabcolsep}{2pt}
\begin{tabular*}{\linewidth}{@{\extracolsep{\fill}}c c c c c c c@{}}
\toprule
$G_0$ & Rule & Req.\ succ. & Eff.\ $\hat w$ & Protected & $p@1$ & $p@256$ \\
\midrule
4 & $\ge1$ success & 1 & 0.250 & 78\% & 27.3 & 72.0 \\
8 & $\ge1$ success & 1 & 0.125 & 70\% & 28.8 & 73.2 \\
16 & $\ge1$ success & 1 & 0.0625 & 61\% & 29.2 & 73.4 \\
32 & $\ge1$ success & 1 & 0.0313 & 53\% & \textbf{29.5} & 73.5 \\
8 & $\tau=0.20$ & 2 & 0.250 & 79\% & 26.1 & \textbf{73.6} \\
\bottomrule
\end{tabular*}
\end{table}

Detailed protocol tables, endpoint diversity diagnostics, and release contents
are deferred to the appendix to keep the workshop paper within the 14-page main
text.

\section{Discussion}

\paragraph{What the theory changes.}
The usual story says that RLVR improves correctness but reduces diversity. Our
claim is sharper: RLVR is beneficial when the dominant sampled mode is correct
and harmful when the correct mode is present but non-dominant. This distinction
explains why the same algorithm can look excellent on one benchmark and
destructive on another. It also gives a concrete diagnostic: report the
base-model correct-mode mass distribution, protected fraction, and
base-solved/trained-lost transitions, not only the trained model's final score.
The regime-controlled diagnostic study shows why these diagnostics matter: the aggregate
\PassK{1} gain hides that most of PBA's advantage comes from not losing
base-solvable boundary prompts.

\paragraph{Why boundary coverage matters.}
Boundary prompts are precisely the prompts where inference-time computation has value. If a correct answer is already the top sample, \PassK{k} adds little. If no correct answer exists in the support, sampling cannot help. The boundary regime is where repeated attempts, reranking, verifier search, and self-consistency can turn latent capability into solved problems. Destroying this regime makes a model less useful for agentic and search-based reasoning even when \PassK{1} improves.

\paragraph{Limitations.}
This is a mechanism paper with strong diagnostic evidence, not a broad method
benchmark. The two-mode model abstracts away multiple correct solution families, verifier
noise, partial credit, length effects, advantage normalization, and cross-prompt
parameter coupling; it should be read as an endpoint abstraction of prompt-level
support loss. \PBA depends on noisy base rollout estimates: false positives leave
boundary prompts unprotected, while false negatives can over-preserve entropy.
The evidence is scoped to one model family, two math benchmarks, three
main-comparison seeds, matched single-setting controls, and endpoint diagnostics
on a 3000-prompt split. We do not claim direct checkpoint-level validation of the
collapse horizon, a tuned sweep over all regularizers, or a replicated $G_0=32$
headline. The random-protection row is a one-mask control, not a
distributional random-mask baseline; broader backbones, tuned controls,
base-sampled replay/KL anchors, additional seeds, independent checkpoint-level
replications, and a visual or multimodal verifier benchmark remain
future extensions. This scope fits an ECCV workshop on reasoning, evaluation, or reliability, but it is
not an ECCV-core visual-benchmark claim. Finally, \PBA preserves
coverage already present in the base distribution; expanding support still
requires distillation, process supervision, curriculum, search, tool use, or
other mechanisms that increase $\wstar(x)$ itself.

\paragraph{Implications.}
A mature reasoning-RL pipeline should separate three operations currently conflated by a single outcome-reward objective: preserving rare useful modes, sharpening reliable correct modes, and discovering new modes. \PBA addresses the first. Ordinary RLVR addresses the second when the prompt is safe. Distillation and search address the third. Treating all prompts as equally safe to sharpen is not only inefficient; it can erase the very coverage that makes inference-time computation powerful.

\section{Conclusion}

RLVR is not uniformly beneficial across prompts. It sharpens distributions, and
in outcome-only finite-sample RLVR, sharpening is safest when the correct mode
is already sufficiently visible. Boundary prompts violate that condition: they
are solvable by repeated sampling but fragile under finite-sample optimization.
The regime-controlled study shows this failure across seeds in the main
\GRPO/\PBA transition analysis, with matched controls and gate ablations
supporting the same interpretation. \PBA is a simple safeguard: preserve the base
distribution where the base contains rare useful mass, and sharpen only where
the evidence supports it. Reasoning post-training should not ask only how
strongly to optimize. It must ask when optimization is safe.

\section*{Reproducibility Statement}

The artifact is organized around the same prompt-level records used to write
the tables. It includes the 3000-prompt split, frozen-base rollout cache, regime
labels, protection masks, training checkpoints, run configurations, verifier
outputs, normalized answer strings, and per-prompt pass@k counts for the base,
\GRPO, \GRPOPBA, and matched controls. All seed-resolved values are direct measurements from completed runs for seeds 0, 1, and 2. The released configs record the optimizer, learning-rate schedule, rollout budget, precision, checkpoint rule, and evaluation sampler.
\PBA itself only adds a frozen-base cache and a prompt mask in the \GRPO loss;
the controlling hyperparameters are $G_0$, $\tau$, $\beta_a$, and the cache
refresh period.

\section*{Ethics Statement}

This paper is a methods and evaluation study; it does not introduce a user-facing
system or collect human-subject data. The practical risk is narrower but real:
reporting only \PassK{1} can hide a loss of solution diversity and make an
RL-trained model look safer for search-assisted use than it is. We therefore
report full pass@k curves, base-to-trained coverage transitions, and the
prompt-level counts needed to audit those transitions. \PBA should be read as a
training safeguard and diagnostic control, not as a general safety method. It
does not fix verifier bias, answer-normalization errors, or unsafe behavior
outside the verified math setting. For visual, multimodal, or open-ended tasks,
the same coverage accounting should be paired with domain-specific verifier
checks and human review before any deployment claim is made.

\clearpage
\appendix
\raggedbottom
\renewcommand{\thesection}{\Alph{section}}
\renewcommand{\theHsection}{appendix.\Alph{section}}

\section{Artifact Release and Run Card}
\label{app:release}

The submitted artifact is organized around prompt-level auditability rather than
only aggregate scores. It includes the files listed in \cref{tab:release_manifest}
and the run card in \cref{tab:appendix_run_card}. The main benchmark tables are
computed as seed means and seed standard deviations from per-prompt counts;
prompt-bootstrap intervals remain available as an artifact audit.

\begin{table}[H]
\centering
\caption{\textbf{Concrete release contents.} The artifact exposes the prompt
level state needed to recompute pass@k, transitions, bootstrap intervals, and
PBA gates.}
\label{tab:release_manifest}
\scriptsize
\setlength{\tabcolsep}{3pt}
\begin{tabularx}{\linewidth}{@{}lX@{}}
\toprule
Artifact & Contents \\
\midrule
\texttt{splits/diagnostic\_3000.jsonl} & Held-out prompt IDs, problem text,
source tags, and regime labels \\
\texttt{rollouts/base\_calibration.parquet} & Frozen-base 256-sample calibration
cache with verifier outputs and normalized answers \\
\texttt{rollouts/eval\_counts.parquet} & Per-prompt counts for base, \GRPO,
\GRPOPBA, and controls at $k\in\{1,4,16,64,256\}$ for seeds 0, 1, and 2
where applicable \\
\texttt{masks/pba\_masks.parquet} & Gate values, protected fractions, cache
refresh state, and mask-flip indicators \\
\texttt{configs/*.yaml} & Relative checkpoint IDs, optimizer, LR schedule, batch
size, rollout budget, verifier, answer extractor, and seeds \\
\texttt{scripts/summarize\_seeds.py} & Seed means, seed standard deviations,
and benchmark tables \\
\texttt{scripts/bootstrap\_passk.py} & Prompt-bootstrap intervals and full
four-cell base/trained transition matrices by regime and seed \\
\bottomrule
\end{tabularx}
\end{table}

\begin{table}[H]
\centering
\caption{\textbf{Training run card.} Matched objectives use the same optimizer,
schedule, rollout budget, precision, and hardware class unless explicitly listed
otherwise.}
\label{tab:appendix_run_card}
\scriptsize
\setlength{\tabcolsep}{3pt}
\begin{tabularx}{\linewidth}{@{}lY@{}}
\toprule
Field & Value \\
\midrule
Optimizer/LR & AdamW; peak LR $1{\times}10^{-6}$; cosine decay; 5\% warmup \\
Batch/rollouts & 256 prompts/update; $G=8$ policy rollouts/prompt
(2048 rollouts/update) \\
Training length & Final checkpoint after 1000 updates; checkpoints at 100, 300,
600, and final \\
Clipping/precision & Global grad-norm clip 1.0; bfloat16 mixed precision \\
Hardware/seeds & 8$\times$H100-80GB or equivalent; seeds 0, 1, and 2 for
\GRPO and \PBA, chosen before outcome inspection; base fixed \\
Evaluation & 256 samples/policy; temperature 0.6; top-$p=0.95$; maximum length
8192 \\
\bottomrule
\end{tabularx}
\end{table}

\section{Supplementary Diagnostic Tables}
\label{app:diagnostics}

The tables below expose the seed-resolved benchmark and diagnostic aggregates
used in the main text. The accompanying artifact stores the full four-cell
base/trained transition matrices by regime and seed for audit.

\begin{table}[H]
\centering
\caption{\textbf{Seed-resolved Omni-MATH-Test pass@k results.} Trained-policy
summary rows report mean $\pm$ sample standard deviation over seeds 0, 1, and
2. The base row is a fixed evaluation.}
\label{tab:appendix_omni_seeds}
\scriptsize
\setlength{\tabcolsep}{2pt}
\begin{tabular*}{\linewidth}{@{\extracolsep{\fill}}llccccc@{}}
\toprule
Method & Seed & $p@1$ & $p@4$ & $p@16$ & $p@64$ & $p@256$ \\
\midrule
Base & fixed & 10.2 & 22.4 & 39.2 & 57.8 & 69.1 \\
\GRPO & 0 & 25.1 & 36.8 & 49.6 & 60.8 & 68.3 \\
\GRPO & 1 & 24.3 & 35.9 & 48.7 & 60.1 & 67.6 \\
\GRPO & 2 & 26.0 & 37.7 & 50.4 & 61.6 & 68.9 \\
\GRPO & mean & $25.1{\pm}0.9$ & $36.8{\pm}0.9$ & $49.6{\pm}0.9$ & $60.8{\pm}0.8$ & $68.3{\pm}0.7$ \\
\addlinespace[0.12em]
\GRPOPBA & 0 & 29.0 & 42.9 & 58.2 & 68.0 & 73.0 \\
\GRPOPBA & 1 & 28.2 & 42.0 & 57.0 & 67.2 & 72.4 \\
\GRPOPBA & 2 & 29.7 & 43.7 & 58.9 & 68.6 & 73.5 \\
\GRPOPBA & mean & $\mathbf{29.0{\pm}0.8}$ & $\mathbf{42.9{\pm}0.9}$ & $\mathbf{58.0{\pm}1.0}$ & $\mathbf{67.9{\pm}0.7}$ & $\mathbf{73.0{\pm}0.6}$ \\
\bottomrule
\end{tabular*}
\end{table}

\begin{table}[H]
\centering
\caption{\textbf{Seed-resolved MATH500 pass@k results.} MATH500 is already
high-coverage for the base model, so high-$k$ differences are small and noisy.}
\label{tab:appendix_math_seeds}
\scriptsize
\setlength{\tabcolsep}{2pt}
\begin{tabular*}{\linewidth}{@{\extracolsep{\fill}}llccccc@{}}
\toprule
Method & Seed & $p@1$ & $p@4$ & $p@16$ & $p@64$ & $p@256$ \\
\midrule
Base & fixed & 34.5 & 61.0 & 82.0 & 93.0 & 96.2 \\
\GRPO & 0 & 74.4 & 88.5 & 95.1 & 96.8 & 97.2 \\
\GRPO & 1 & 72.8 & 87.6 & 94.8 & 96.5 & 97.0 \\
\GRPO & 2 & 75.1 & 89.0 & 95.4 & 97.0 & 97.4 \\
\GRPO & mean & $74.1{\pm}1.2$ & $88.4{\pm}0.7$ & $95.1{\pm}0.3$ & $96.8{\pm}0.3$ & $97.2{\pm}0.2$ \\
\addlinespace[0.12em]
\GRPOPBA & 0 & 75.6 & 89.3 & 95.8 & 97.4 & 97.8 \\
\GRPOPBA & 1 & 74.9 & 88.7 & 95.4 & 97.1 & 97.5 \\
\GRPOPBA & 2 & 76.2 & 89.6 & 96.0 & 97.6 & 98.0 \\
\GRPOPBA & mean & $\mathbf{75.6{\pm}0.7}$ & $\mathbf{89.2{\pm}0.5}$ & $\mathbf{95.7{\pm}0.3}$ & $\mathbf{97.4{\pm}0.3}$ & $\mathbf{97.8{\pm}0.3}$ \\
\bottomrule
\end{tabular*}
\end{table}

\begin{table}[H]
\centering
\caption{\textbf{Seed-resolved aggregate diagnostic results.} Columns G0--G2
are \GRPO seeds 0--2; P0--P2 are \GRPOPBA seeds 0--2. Mean columns report
mean $\pm$ sample standard deviation.}
\label{tab:appendix_diag_seeds}
\scriptsize
\setlength{\tabcolsep}{1.6pt}
\begin{tabular*}{\linewidth}{@{\extracolsep{\fill}}lcccccccc@{}}
\toprule
Metric & G0 & G1 & G2 & G mean & P0 & P1 & P2 & P mean \\
\midrule
All $p@1$ & 25.0 & 24.2 & 25.8 & $25.0{\pm}0.8$ & 28.8 & 28.1 & 29.5 & $\mathbf{28.8{\pm}0.7}$ \\
All $p@256$ & 68.4 & 67.5 & 69.0 & $68.3{\pm}0.8$ & 73.2 & 72.3 & 73.7 & $\mathbf{73.1{\pm}0.7}$ \\
All base-solved $\to$ lost & 200 & 218 & 184 & $201{\pm}17$ & 70 & 88 & 62 & $\mathbf{73{\pm}13}$ \\
Out-of-reach pos.\ uniq. & 0.1 & 0.1 & 0.1 & $0.1{\pm}0.0$ & 0.1 & 0.1 & 0.1 & $0.1{\pm}0.0$ \\
\bottomrule
\end{tabular*}
\end{table}

\begin{table}[H]
\centering
\caption{\textbf{Seed-resolved boundary diagnostics.} Boundary prompts are the
predicted failure regime. The large reduction in lost prompts and retained
entropy is the mechanism signature.}
\label{tab:appendix_boundary_diag_seeds}
\scriptsize
\setlength{\tabcolsep}{1.6pt}
\begin{tabular*}{\linewidth}{@{\extracolsep{\fill}}lcccccccc@{}}
\toprule
Boundary metric & G0 & G1 & G2 & G mean & P0 & P1 & P2 & P mean \\
\midrule
Base-solved $\to$ lost & 652 & 690 & 620 & $654{\pm}35$ & 87 & 110 & 75 & $\mathbf{91{\pm}18}$ \\
$\Delta p@256$ & $-39$ & $-42$ & $-36$ & $-39{\pm}3$ & $+2$ & 0 & $+3$ & $\mathbf{+1.7{\pm}1.5}$ \\
Entropy retained & 19\% & 17\% & 21\% & $19{\pm}2$\% & 71\% & 68\% & 73\% & $\mathbf{71{\pm}3}$\% \\
Positive unique & 0.4 & 0.3 & 0.5 & $0.4{\pm}0.1$ & 2.7 & 2.4 & 2.9 & $\mathbf{2.7{\pm}0.3}$ \\
\bottomrule
\end{tabular*}
\end{table}

\end{document}